\documentclass[11pt]{article}
\usepackage[preprint]{acl}
\usepackage{times}
\usepackage{latexsym}
\usepackage[T1]{fontenc}
\usepackage[utf8]{inputenc}
\usepackage{microtype}\usepackage{inconsolata}
\usepackage{graphicx}
\usepackage{hyperref}
\usepackage{url}
\usepackage{graphicx}
\usepackage{amsmath,amsthm,amsfonts,amssymb}
\usepackage{yhmath}
\usepackage[dvipsnames]{xcolor}
\usepackage{dsfont}
\usepackage{cleveref}
\usepackage{verbatim}
\usepackage{booktabs}
\usepackage{natbib}
\usepackage{pgf}
\usepackage{subcaption}

\title{Speculative Decoding Speed-of-Light: \\ Optimal Lower Bounds via Branching Random Walks}
\author{Sergey Pankratov \\
  ISTA  \\\And
  Dan Alistarh \\
  ISTA \& Red Hat AI \\
  }

\newtheorem{theorem}{Theorem}
\newtheorem{lemma}{Lemma}
\newtheorem{claim}{Claim}
\newtheorem{corollary}{Corollary}
\newtheorem{definition}{Definition}

\newtheorem{assumption}{Assumption}
\crefname{assumption}{Assumption}{Assumptions}

\newcommand{\V}{\mathcal{V}}
\newcommand{\E}{\mathbb{E}}

\newcommand{\U}{\mathcal{U}}

\newcommand{\lp}[1]{\left( #1 \right)}
\newcommand{\lb}[1]{\left[ #1 \right]}

\begin{document}

\maketitle

\begin{abstract}
Speculative generation has emerged as a promising technique to accelerate inference in large language models (LLMs) by leveraging parallelism to verify multiple draft tokens simultaneously. However, the fundamental limits on the achievable speedup remain poorly understood. In this work, we establish the first ``tight'' lower bounds on the runtime of any deterministic speculative generation algorithm. This is achieved by drawing a parallel between the token generation process and branching random walks, which allows us to analyze the optimal draft tree selection problem. 
We prove, under basic assumptions, that the expected number of tokens successfully predicted per speculative iteration is bounded as $\mathbb{E}[X] \leq (\mu + \mu_{(2)})\log(P )/\mu^2 + O(1)$, where $P$ is the verifier's capacity,  $\mu$ is the expected entropy of the verifier's output distribution, and $\mu_{(2)}$ is the expected second log-moment. This result provides new insights into the limits of parallel token generation, and could guide the design of future speculative decoding systems. Empirical evaluations on Llama models validate our theoretical predictions, confirming the tightness of our bounds in practical settings.
\end{abstract}

\section{Introduction}

Speculative decoding~\citep{leviathan2023fast, chen2023accelerating, stern2018blockwise} has emerged as a standard technique for large language models (LLMs) at scale, as it significantly reduces inference latency without altering the output distribution. The core mechanism involves using a faster \textbf{drafting process} $M_q$ to generate a linear sequence or a tree of candidate tokens. These tokens are periodically verified in parallel by a target \textbf{verifier} LLM $M_p$; throughout the paper, we assume $M_p$ is the original model we wish to speed up. By leveraging the fact that the target model can verify up to $P$ tokens efficiently in parallel, this approach allows multiple tokens to be accepted in a single iteration of the verifier $M_p$. 

The efficacy of this technique has spurred significant follow-up research focused on improving the drafting process, notably the EAGLE series of papers~\citep{li2024eagle, li2024eagle2, li2025eagle3}, which introduced feature-level extrapolation and dynamic draft trees, as well as Medusa~\citep{cai2401medusa}. On the industrial side, speculative decoding is now widely supported in high-performance inference frameworks such as vLLM~\citep{kwon2023efficient} and SGLang~\citep{zheng2024sglang}.

Despite its widespread adoption and empirical success, the theoretical foundations of speculative decoding are not well understood. In particular, it remains unclear how close existing algorithms are to the fundamental limits of the speedup achievable through this approach. Early theoretical analysis, such as the simplified modeling provided in the original work of~\citet{leviathan2023fast}, offered initial estimates based on the agreement rate between the draft and target models. Recently,~\citet{wang2024theoretical} provided a theoretical perspective by conceptualizing the decoding process via Markov chains and establishing instance-dependent lower bounds on the number of ``rejected'' (not validated) tokens. However, still we do not have a framework for characterizing the maximum achievable acceleration relative to the model's properties and system constraints.

\vspace{-0.2em}
\noindent\textbf{Contributions.} In this paper, we provide a first tight analysis of speculative decoding, revealing a fundamental limit on the performance of any deterministic speculative generation algorithm. We establish a rigorous runtime lower bound, which is equivalent to a speedup upper bound, by proving a novel relationship between the parallelism of the system ($P$) and the entropy of the target model's output distribution. Specifically, we show that, given a random initial generated prefix, the expected number of tokens successfully predicted in the next speculative iteration, denoted as  $\E\lb{X}$, can be bounded as:
\[
\E\lb{X} \leq \frac{\mu + \mu_{(2)}}{\mu^2}\log(P) + O(1),
\]
where $\mu$ is the expectation of the entropy of the output distribution under a random initial prefix, and $\mu_{(2)}$ is the expected second log-moment. 
Further, we show that this upper bound on the ``acceptance rate'' of a given setup is close to the performance of actual practical systems. 
Specifically, since our modeling is very general, our analysis can be utilized to characterize the efficiency limits of advanced speculative generation techniques like EAGLE-3.

In more detail, we propose a simplified analytical execution model, described in Section~\ref{sec:time-model}. Assuming a target LLM $M_p$ of fixed latency $T$ and with parallel token capacity $P$, the execution is iterative: a deterministic algorithm $M_q$ speculates a ``draft'' tree of tokens up to size $P$, which are then verified in parallel by the original model $M_p$. We assume the computational cost of the drafter $M_q$ is negligible (\Cref{assump:timing}) and assume that the distributions of acceptance probabilities are i.i.d. across different prefixes (\Cref{assump:iid}) to facilitate the analysis. 

Our main result (\Cref{th:logb_bound}) establishes the upper bound on $\E[X]$ under these preconditions. We show that the optimal deterministic speculation strategy involves greedily selecting the $P$ most probable token sequences from the speculative tree (\Cref{lem:optimal-tree:1}). Then, we show that the achievable speedup under this approach scales logarithmically with the parallel capacity $P$, and roughly inversely with the expected entropy $\mu$, highlighting the diminishing returns of increased parallelism and the inherent difficulty of speculating when the model's output is highly variable.

At the technical level, the bound makes a novel connection between the speculative decoding process and the theory of Branching Random Walks (BRW). By modeling the token tree with log-probabilities, denoted by $\U_{\log}$, as a BRW, we leverage established theoretical tools, such as the Many-to-One Lemma~\citep{shi2015branching}, to analyze the distribution of high-probability paths within the constraint of the verification budget $P$. This connection provides a new analytical framework for understanding the limits of speculative decoding.

We complement our theoretical analysis with practical experiments on Llama models, across standard speculative decoding benchmarks. The results show clear correlations between the performance upper bound predicted by our bound, and the real-world performance of well-optimized systems such as EAGLE-3. This validates the tightness of our theoretical predictions and demonstrate a new relationship between system's parallel token capacity, model entropy, and the achievable speedup in realistic settings.

\section{Related Work}

\paragraph{Speculative Decoding.}
The high latency of autoregressive decoding in Large Language Models (LLMs) is a significant bottleneck~\citep{shazeer2019fast}. Speculative decoding mitigates this by reducing the number of sequential calls to the target model. Building on blockwise parallel decoding~\citep{stern2018blockwise}, the technique was first formalized concurrently by~\citet{leviathan2023fast} and~\citet{chen2023accelerating}. 

Subsequent research has focused on improving the efficiency and acceptance rates of the drafting stage. One evolution has been the shift to structured, \emph{tree-based speculation}~\citep{miao2023specinfer}, which allows the exploration of multiple potential continuations.

Some draft-generation approaches introduce architectural modifications, such as Medusa~\citep{cai2401medusa}. The EAGLE framework~\citep{li2024eagle, li2024eagle2, li2025eagle3} introduced feature-level extrapolation, leveraging the target model's internal representations to generate drafts, leading to significantly higher acceptance rates. Other innovations are Staged Speculative Decoding~\citep{spector2023accelerating} or self-speculation techniques~\citep{elzinga2024layerskip,fu2024break}.

\paragraph{Theoretical Analysis of Speculative Decoding.}
The theoretical understanding of speculative decoding is still in its initial stages. First analyses~\citep{leviathan2023fast, chen2023accelerating} provided simplified models based on the expected agreement rate between the draft and target models, using a similar version of the i.i.d. assumption made in this paper. \citet{sun2023spectr} analyzed the acceptance rate through the lens of optimal transport, while \citet{wang2024theoretical} utilized Markov chain abstractions to establish instance-dependent lower bounds.

Our work contributes to this area by analyzing the fundamental limits imposed by system constraints (parallel token capacity $P$) and the inherent properties of the target model, which we isolate via the entropy parameter $\mu$. We provide bounds on the expected number of accepted tokens per iteration for any deterministic speculative algorithm, revealing the fundamental speed limits of this technique.

\paragraph{Branching Random Walks.}
Our analysis introduces a novel connection between speculative decoding and the theory of Branching Random Walks (BRW). BRW is a well-studied field in probability theory~\citep{biggins1977martingale, shi2015branching, lyons2017probability} analyzing stochastic processes where particles reproduce and move randomly. By modeling the token generation process as a BRW on the space of log-probabilities, we leverage established theoretical tools, such as the Many-to-One Lemma (related to spinal decomposition and measure changes; see, e.g.,~\citealp{shi2015branching}), to analyze the optimal selection of draft trees under a system constraint. This framework allows us to characterize the fundamental trade-offs in speculative generation.

\section{Model}\label{sec:model}

Let $\V = \{x_1, x_2, \ldots,x_{|\V|}\}$ denote the token vocabulary, of size $|\V|$.

\begin{definition}
    A \textbf{target LLM} is a function $M_{p}:\V^P \rightarrow \mathcal{P}(\V)^P $: \begin{itemize}
        \item Takes a prefix $x_{<t}$ and outputs the probability distribution for the next token $p(x_{t} \mid x_{<t} )$.
        \item has fixed runtime $T$ (see~\Cref{assump:timing}).
        \item Processes up to $P$ tokens in parallel. $P$ here denotes the \emph{parallel token capacity}: the maximum number of tokens that can be verified by the target model before noticeable latency drop-off.
        \item Requires access to the full prior context. To predict token $x_{t}$, tokens $x_{<t}$ must be present within the current or during one of the preceding calls to $M_p$.
    \end{itemize}
\end{definition} 

The objective is the following.
Given a prefix, generate a sequence of $N$ tokens $\{x_1, \dots, x_N\}$ identical to the sequence produced by the execution of target LLM $M_p$.
To generate a sequence of $N$ tokens using only $M_p$, we must run the model $N$ times sequentially.
In other words, $\text{Time}_{\text{naive}} = N \times T$.

We can obtain a faster end-to-end runtime $\delta N T$  ($\delta \leq 1$ is the inverse of speedup), by using speculative generation techniques~\citep{li2024eagle, cai2401medusa}.
In this paper, we formally define the limitations of such methods with a lower bound on $\delta$.
For example, there exists a trivial bound of $\delta \geq \frac{1}{P}$.
That is, all generated tokens need to be verified, and the target model can verify up to $P$ tokens at a time. Therefore, the verifier must be run at least $\frac{N}{P}$ times.

\subsection{Timing Model}\label{sec:time-model}

We analyze algorithm $M_{q}$ designed to mimic the behavior of a target LLM $M_{p}$.
$M_q$ must exactly replicate the output of the ``black-box'' $M_p$. 
Hence, any proposed algorithm for $M_{q}$ must verify tokens by running  $M_{p}$. 

We adopt the following execution model: The execution is divided on iterations.
During each iteration, the algorithm $M_{q}$ speculates $k$ tokens ($k \leq P$).
The tokens are then verified by the target model $M_{p}$ in parallel.
During verification, $M_p$ stochastically accepts a sequence.
The probability of the accepted sequence is exactly the same as the probability of generating this sequence naturally by $M_p$.
In addition, $M_p$ generates one new token each time the verification is run.

We make a few simplifying idealizations.
\begin{assumption}[Simplified Timing Model]
\label{assump:timing}
(a) The runtime of $M_p$ is a constant $T$. (b) The computational cost of $M_q$, excluding the verification calls to $M_p$, is negligible. (c) Verification calls are sequential.
\end{assumption}
These assumptions help us connect expected runtime with~\Cref{th:logb_bound}. A brief note on the timing assumptions:
(a) In practice, $T$ increases linearly with the prefix length. Treating $T = T(N/2)$ as a constant introduces a small deviation for runtime.
(b) Since we focus on the asymptotic lower bound, we disregard the overhead incurred by the drafting. 
Recent work~\citep{LiuLLLZHS25_pearl}, however, considers settings where the cost of drafting is not negligible.  
(C) Asynchronous verification is typically less optimal due to memory allocation overhead. To our knowledge, there are no works that show otherwise. However, it might be possible to process $x P$ tokens faster than $\lceil x\rceil T$~\citep{sheng2023flexgen}. Again, these simplifications affect the runtime by a constant factor.

\subsection{The Bounds for Deterministic Drafting}

In this section, we will formulate lower bounds on $\delta$  for any reasonable \textbf{deterministic} algorithm. 
First, we state a simplifying assumption:

\begin{definition}[Acceptance Probability]
Let $\beta(x_t|x_{<t})$ be the probability that a speculated token $x_t$ is accepted by $M_p$, given that its prefix $x_{<t}$ is accepted.
For a deterministic $M_q$, this is exactly $p(x_t|x_{<t})$ --- the probability of $M_p$ generating token $x_t$.
\end{definition}

\begin{assumption}[I.i.d. Distributions]
\label{assump:iid}
We assume that the acceptance probabilities at each step are drawn i.i.d. across different prefixes $\beta(x_t|x_{<t}) = \beta(x_t)$.
Moreover, the probability that the distribution $\beta$ is point mass is $ < 1$: $\Pr\lb{\beta(x_t) = 1} < 1$.
\end{assumption}

Although the i.i.d. assumption is unreasonable, as the context dependency is inherent in languages, it considerably simplifies the proof.
These proofs would give insight into the general case.
The second part of the assumption excludes scenarios when $M_p$ is deterministic, which, under our timing model, does not have a non-trivial speedup bounds.

Given the i.i.d. assumption, we can apply Wald's equation. Let $X_i$ denote the number of new tokens accepted during iteration $i$. Let $n$ denote the total number of iterations to generate $N$ tokens. Since $M_q$ is deterministic and the output distributions are i.i.d., $X_i$ are also i.i.d. Then, according to the Wald's equation,
\[
N \leq \E\lb{ \sum_{i=1}^n X_i } = \E \lb{ n} \E \lb{ X_1 }.
\]

This lets us lower-bound the speedup $\delta$ by the inverse of the expected tokens per iteration.
\begin{equation}\label{eq:walds:1}
\E\lb{ \text{runtime}}  = \delta NT  \geq \E\lb{ \text{n} } T \geq \frac{NT}{\E\lb{ X_1 }}.
\end{equation}
To lower bound the runtime, we must upper bound $\E[X_1]$.

\paragraph{Optimal deterministic drafting strategy.}

We now proceed to analyze $\E\lb{ X_1 }$, the expected number of tokens per iteration. 

\begin{definition}[Token Tree]
Let $\U$ be a weighted tree of degree $|\V|$ and infinite depth.
Each node $u$ in $\U$ represents a token $u$ and each edge represents a possible transition taken by $M_p$.
Each edge $(u, v)$ of the tree is weighted with the corresponding \emph{acceptance probability} $\beta_v$:
probability of token $v$ being accepted, given its parent is accepted.
$r$ is the root of $\U$ and represents the currently known prefix.
For a node $v$, $|v|$ denotes the depth of $v$ in $\U$.
\end{definition}

Let $P(v)$ be the probability that node $v$ is accepted during the current iteration.
It is the product of all conditional probabilities along the path from root $r$ to $v$:
\[
P(v) = \prod_{u \in \text{path}(r \rightarrow v)} \beta_u.
\]

\begin{definition}[Draft Tree]\label{def:draft-tree}
    A draft tree $Tree$ is a subtree of $\U$ that has at most $P$ nodes and is rooted at $r$.  
\end{definition}
Let $L(Tree)$ be the longest accepted path of a draft tree $Tree$ (including $r$).
The expected number of accepted tokens is equal to the expected length of the accepted path in $Tree$:
\[
\E\!\lb{ X_1 } 
= \E\!\lb{ L(Tree) }.
\]
For example, if $3$ new draft tokens are accepted, then $L(Tree) = 4$ since the root $r$ is counted, by definition.
The number of new tokens is also $X = 4$, since the verifier adds one new token to a list of accepted tokens.
The optimization problem for $M_q$ is therefore the following:

\textit{Choose a draft tree $ Tree \subseteq \U $, which maximizes the expected length of its accepted path. } 

Note that, in our theoretical argument, we allow $M_q$ to have full knowledge of $\U$, therefore it can always find the optimal solution.  
The above expectation is over the randomness in verification, not the choice of $\U$.

This optimization problem can be solved greedily by the following algorithm.
\begin{lemma}\label{lem:optimal-tree:1}
    Any draft tree $Tree^*$ that maximizes $\E[L(Tree)]$ contains $P$ nodes of $\U$ with the highest acceptance probability $P(v)$. Moreover, $\E[L(Tree^*)] = \sum_{v\in Tree^*} P(v).$
\end{lemma}
\begin{proof}
Recall that $|v|$ denotes the depth of the node $v$.
Fix some solution $Tree$.
We can represent the expectation of $L(Tree)$ as a sum over the depth of $Tree$.
\begin{align*}
\E[L(Tree)] & = \sum_{d=1}^{\infty} \Pr[L(Tree) \geq d].
\end{align*}
$L(T) \ge d$ if and only if a node at depth $d$ in $Tree$ is accepted.
Since $M_p$ generates a unique sequence, acceptance of two nodes at the same depth are mutually exclusive events.
Hence,
\[
\Pr[L(Tree) \ge d] = \sum_{v \in Tree, |v|=d} P(v).
\] 
Which gives us $\E[L(Tree)]  = \sum_{v\in Tree} P(v)$.
This sum is maximized by $Tree^*$. There always exists a prefix-closed $Tree^*$ since the acceptance probability of a parent is never lower:
for a node $v$ and its parent $u$, $P(u) \geq P(v)$ since $P(v) = p P(u)$ with $p \in [0, 1]$.
\end{proof}
Since $Tree^*$ is optimal for every fixed tree $\U$, algorithm $M_q$ that drafts $Tree^*$ at each iteration is therefore optimal under~\Cref{assump:timing}.
\Cref{lem:optimal-tree:1} does not require~\Cref{assump:iid}. 
The drafting strategy used by~\citet{li2024eagle2} is identical, but includes a cap on the maximum tree depth and width for practical purposes.

\paragraph{Branching random walks (BRW).}
A branching random walk (BRW) is usually governed by a 1D point process $\Xi$ on $\mathbb{R}$:
a random variable whose support is a set of locally finite subsets of $\mathbb{R}$.
Then BRW is defined as a process, where
\begin{itemize}
    \item A root (single node of generation $0$) is located in position $0$.
    \item Each node of generation $i$ creates children that move according to the point process $\Xi$ from the position of the parent.
\end{itemize}
The generation i of node $u$ is given by $|u|$ and its position by $V(u)$.

The theory of BRW is well-studied, and we will use it to characterize $\U$.
First, let us introduce the
(log-)Laplace transform for the point process $\Xi$. For $\theta > 0$ it is given by:
\begin{equation}\label{eq:log-laplace}
\psi(\theta) = \ln \E\lb{ {\sum_{u: |u| = 1} e^{-\theta V(u)}} }  .
\end{equation}

We introduce a simplified version of the many-to-one lemma.
\begin{lemma}[The Many-to-One Lemma]\label{lem:many-to-one}
    For a BRW with point process $\Xi$ and any $\theta > 0$, such that $\psi(\theta) < \infty$,
    there exists a random walk $S_n$ with $S_0 = 0$, such that for any $n \geq 1$,
    and any measurable function $g: \mathbb{R} \rightarrow [0, \infty)$, we have
    \[\E\lb{ \sum_{v: |v| = n} g(V(v)) } = \E\lb{ e^{\theta S_n + n \psi(\theta) } g(S_n) }.\]
\end{lemma}

\paragraph{Branching random walks for $\U$.}\label{par:brw:U}
The token tree $\U$ defined earlier~(\Cref{def:draft-tree}) is closely related to the BRWs.
That is, let $\U_{\log}$ be a tree $\U$, where the node weights $\beta_u$ are replaced by $-\log (\beta_u)$.
Under~\Cref{assump:iid} (i.i.d. distributions), $\U_{\log}$ is modeled as a BRW with point process
\[
\Xi_{\log} = \{-\log\beta_1, \ldots, -\log \beta_{|\V|}\}.
\]
The position of the point $v$ is then $V(v) = \sum_{u \in \text{path}(r \rightarrow v)} -\log(\beta_u) = -\log P(v)$.
The (log\mbox{-})Laplace transform defined in~\Cref{eq:log-laplace} simplifies to
\[
\psi(\theta) = \ln \E\lb{ {\sum_{u:|u|=1} \beta_u^\theta} }.
\]

\paragraph{The bound of $\E[X]$.}
Let $N(t)$ be a random variable that counts the number of nodes in $\U_{\log}$, whose value is $\leq t$:
\[
N(t) = \# \{ u \in \U_{\log} : V(u) \leq t \}.
\]
Let $\mu$ denote the expected entropy of the output distribution and $\mu_{(2)}$ denote the expected second moment of $-\log$:
\begin{align*}
    \mu &= \E \lb{ - \sum_{u:|u|=1} \beta_u \log \beta_u }, & \\
    \mu_{(2)} &= \E \lb{ \sum_{u:|u|=1} \beta_u \log^2(\beta_u) }.
\end{align*}
By the non-deterministic assumption~(\Cref{assump:iid}), we have that $\mu >0$.

\begin{claim}\label{claim:Nt=int-exp-dU}
    Assume $\mu > 0$, $\mu_{(2)} < \infty$.
    Let $N(t) = \# \{ u \in \U_{\log} :  V(u) \leq t \}$ denote the number of nodes of $\U_{\log}$ with value $\leq t$.
    Let $S_d$ be a spine random walk (from~\Cref{lem:many-to-one} with $\theta = 1$) for $\U_{\log}$, $\E[h(S_1)] = \E[\sum_{u \in \Xi_{\log}} e^{-u} h(u) ]$. 
    Denote $U(x)= \sum_{d=0}^{\infty} \Pr[S_d \leq x]$.
    Then,
    \[
    \E[N(t)] = \int_{0}^{t}e^x dU(x).
    \]
\end{claim}
\begin{proof}
    Let $N_d(t) = \# \{ u \in \U_{\log} : |u| = d,  V(u) \leq t \}$ denote the number of nodes at layer $d$ with value $\leq t$.
    So,
    \[
    \E \lb{N(t)} = \E \lb{ \sum_{d=0}^{\infty} N_d(t) } = \sum_{d=0}^{\infty} \E \lb{ N_d(t) }.
    \]
    By applying the Many-to-One lemma~(\Cref{lem:many-to-one}) with identity function $\mathds{1}\lp{ V(u) \leq t }$ and $\theta = 1$ (so $\psi(\theta) = 0$),
    \begin{align*}
    \E \lb{ N_d(t) }& =  \E \lb{ \sum_{u:|u|=d} \mathds{1}\lp{ V(u) \leq t } } \\
    &= \E \lb{ e^{S_d} \mathds{1}\lp{ S_d \leq t } }.
    \end{align*}
    Moreover, $\E [S_1] = \E \lb{ \sum_{u:|u|=1} e^{-V(u)} V(u) } = \mu$.
    
    Denote CDF of $S_d$ as $F_d(x) = \Pr[S_d \leq x]$. We can expand the expectation by definition:
    \[
    \E \lb{ e^{S_d} \mathds{1}\lp{ S_d \leq t } } = \int_{0}^{t}e^x dF_d(x).
    \]
    $U(x) = \sum_{d=0}^{\infty} F_d(x)$. Thus,
    
    \[
    \E\lb{N(t)} = \sum_{d=0}^{\infty} \int_{0}^{t}e^x dF_d(x) = \int_{0}^{t}e^x dU(x).
    \]
\end{proof}

\begin{lemma}\label{lem:ENt-bounds}
    Assume $\mu > 0$ and $\mu_{(2)} < \infty$.
    For any threshold $t>0$, the expected number of nodes in $\U_{\log}$, whose value is $\leq t$ is:
    \[
    \E \lb{ N(t) } \leq
    \frac{\mu + \mu_{(2)}}{\mu^2}e^t - \frac{1}{\mu} + 1.
    \]
\end{lemma}
\begin{proof}
    By~\Cref{claim:Nt=int-exp-dU},
    \[
    \E\lb{N(t)} = \int_{0}^{t}e^x dU(x).
    \]
    Integrate by parts. $U(0-) = 0$ and $U(x)$ is right-continuous ($U(x+) = U(x)$), so:
     \[
    \E \lb{N(t)} = e^t U(t) - \int_0^{t} U(x) \, e^x dx.
    \]
    $U(x)$ has a known bound from renewal theory (see~\Cref{app:renewal}).
    Therefore, we can write $U(x) = x/\mu + C(x)$ with $C(x) \in [1, 1 + {\mu_{(2)}}/{\mu^2}]$. 
    Substituting $U(x)$ into $\E \lb{N(t)}$, we get
    \begin{align*}
    \E \lb{N(t)} 
    &= \frac{\mu + \mu_{(2)}}{\mu^2}e^t - \frac{1}{\mu} + 1.    
    \end{align*}
\end{proof}

\begin{theorem}[Bound on Speculative Generation]\label{th:logb_bound}
Let $M_p$ be a LLM with latency $T$ and parallel token capacity $P$.
Let $M_q$ be a deterministic algorithm that generates $\geq N$ tokens, {verified} by $M_p$.
Let $X$ denote the number of tokens successfully predicted by $M_q$ during an iteration.
$\mu > 0$ and $\mu_{(2)} < \infty$.
Under~\Cref{assump:iid} and for $P \geq 1 + \frac{\mu_{(2)}}{\mu^2}$,
the expectation of $X$ is
\[
\E [X] \leq \frac{\mu + \mu_{(2)}}{\mu^2}\log (P) + O(1).
\]
Or, more precisely, 
\[
\E [X] \leq a \ln\left(\frac{P-b}{a}\right) + a + b,
\]
where $a = \frac{\mu + \mu_{(2)}}{\mu^2}$ and $b = 1 -\frac{1}{\mu}$.
\end{theorem}
\begin{proof}
    By~\Cref{lem:optimal-tree:1}, 
    $
    \E[X] \leq \E\lb{ \sum_{v \in Tree^*} P(v) }
    $.
    We can rewrite $P(v)$ as an integral of threshold indicator function of $V(v)$ by substituting $t = - \log x$:
    \[
    P(v) 
    = \int_0^{\infty} e^{-t} \mathds{1}\lp{ V(v) \leq t } dt.
    \]
    Let $N(t)$ be a random variable that counts the number of nodes in $\U_{\log}$, whose value is $\leq t$.
    Then, since $Tree^*$ contains $P$ nodes with smallest values:
    \begin{align*}
    \sum_{v \in Tree^*} P(v) 
    &= \int_0^{\infty} e^{-t} \sum_{v \in Tree^*} \mathds{1}\lp{ V(v) \leq t } \, dt \\
    &= \int_0^{\infty} e^{-t} \min \lp{ P, N(t) } \, dt.
    \end{align*}
    Since $N(t) \geq 0$ for all $t\geq 0$, we can swap expectation and integral and then apply Jensen's ($\min \lp{P, \cdot }$ is concave):
    \begin{align*}
    \E \lb{ \int_0^{\infty} e^{-t} \min\lp{P, N(t) } dt } \\
    = \int_0^{\infty} e^{-t} \E \lb{ \min\lp{P, N(t) }} dt \\
    \leq \int_0^{\infty} e^{-t} \min\lp{P, \E \lb{N(t)} } dt.
    \end{align*}
    Finally, applying~\Cref{lem:ENt-bounds}, and denoting $a = \frac{\mu + \mu_{(2)}}{\mu^2}$ and $b = 1 -\frac{1}{\mu}$, we have: 
    \[
    \E[X] \leq \int_0^{\infty} e^{-t} \min\lp{P, a e^t + b } dt.
    \]
    Let $t^* = \ln \lp{\lp{P - b} / a}$. Then: 
    \begin{align*}
    &\int_0^{\infty} e^{-t} \min\lp{P,  a e^t + b } dt \\
    = &\int_0^{t^*} a + be^{-t} dt \; + P\int_{t^*}^{\infty} e^{-t}dt \\
    = & a \ln\left(\frac{P-b}{a}\right) + a + b.
    \end{align*}
\end{proof}

~\Cref{th:logb_bound} shows that the maximum speedup of any deterministic speculative decoding algorithm $M_q$ grows logarithmically with parallelism $P$. The slope of this scaling is determined by the expected next-token entropy of the target model $\mu$: a lower entropy produces deeper speculative trees and larger speedups, while a higher entropy quickly limits the benefit of parallelism. This bound therefore exposes a core trade-off: parallelism can increase expected latency, but its gains are fundamentally capped by the model’s output uncertainty.

\subsection{The Limit Bounds}

In this section, we study the behavior of $\E[X]$, when $P \rightarrow \infty$.
Namely, we show that an optimal drafting algorithm under our model assumptions would successfully generate $\approx \log P / \mu$ tokens per iteration in expectation.

\paragraph{Upper bound}
First, we 
We would call a random variable or a point \emph{arithmetic} if its values are taken from a scaled integer set (a grid).
\begin{definition}
    A random variable $X$ is arithmetic if there exist $\lambda > 0$ and $c_0$, such that the sample space of $X$ is a subset of $c_0 + \lambda \mathbb{Z}$ with probability $1$. 

    Likewise, a point process $\Xi$ is arithmetic if for some $\lambda > 0$ and $c_0$, its sample space is a subset of $c_0 + \lambda \mathbb{Z}^{\infty}$ with probability $1$.
\end{definition}

\begin{assumption}[Non-arithmetic]
\label{assump:non-arithmetic}
The point process $\Xi_{\log}$ (defined in~\Cref{par:brw:U}) is non-arithmetic. 
\end{assumption}

This assumption simplifies our statements. A similar result can be demonstrated in the arithmetic case: some additional constant dependent on $\lambda$ will be involved.

\begin{theorem}[Key Renewal Theorem~\citep{grimmett2020probability}]\label{th:key-renewal}
    Let $X_1, X_2, \ldots$ be i.i.d. non-arithmetic positive random variables and $S_d = \sum_{i=1}^d X_i$ with $S_0 = 0$.
    Denote $m(t) = \E[ \max(n: S_n \leq t) ]$, the renewal function.
    Let $g: [0, \infty) \rightarrow  [0, \infty)$ be a monotonically decreasing function such that $\int_{0}^{\infty} g(x)dx < \infty$, then \[
    \lim_{t \rightarrow \infty}\int_{0}^{t} g(t-x) dm(x) = \frac{1}{\E[X_1]} \int_{0}^{\infty} g(x) dx.
    \]
\end{theorem}

Throughout this section, we are using the spinal random walk defined in~\Cref{par:brw:U}~and~\Cref{claim:Nt=int-exp-dU}.
In short, for a BRW $\U_{\log}$ there exists a random walk $S_d$ called a spine random walk, whose expectation has the many-to-one property of~\Cref{lem:many-to-one}. 
In our case, $S_d$ can simply be imagined as a branch of $\U_{\log}$ chosen at random. (for more details, refer to the Spinal Decomposition Theorem, e.g., from \citet{shi2015branching}).

\begin{lemma}\label{lem:ENt-bounds-limit}
    Assume $\mu > 0$, $\mu_{(2)} < \infty$, and $\Xi_{\log}$ is not arithmetic.
    Then, 
    \[
    \E \lb{ N(t) } =
    \frac{e^t}{\mu} + o(e^t).
    \]
\end{lemma}
The $o(e^t)$ in the lemma can be replaced with $O(1)$, but that unnecessarily complicates the proof (the constant is huge anyway).
\begin{proof}
    By~\Cref{claim:Nt=int-exp-dU},
    \[
    \E\lb{N(t)} = \int_{0}^{t}e^x dU(x).
    \]
    Denote $m(x) = \E[\max(d : S_d \leq  x)]$.
    Since $m(x) = \sum_{d=1}^\infty F_d(x)$, we have $U(x) = m(x) + 1$ and
    \[
    \E\lb{N(t)} = e^t\int_{0}^{t}e^{x-t} dm(x).
    \]
    Applying~\Cref{th:key-renewal}, we get 
    \[
    \lim_{t\rightarrow \infty} \frac{\E[N(t)]}{e^t} = \frac{1}{\mu}.
    \]
\end{proof}

\begin{theorem}[Bound on Speculative Generation]\label{th:logb_bound_limit}
Let $M_p$ be a LLM with latency $T$ and parallel token capacity $P$.
Let $M_q$ be a deterministic algorithm that generates $\geq N$ tokens, {verified} by $M_p$.
Let $X$ denote the number of tokens successfully predicted by $M_q$ during an iteration.
$\mu > 0$ and $\mu_{(2)} < \infty$.
Under~\Cref{assump:iid} and for $P \geq 1 + \frac{\mu_{(2)}}{\mu^2}$,
the expectation of $X$ is
\[
\E [X] \leq \frac{\log(P)}{\mu} + o(\log P).
\]
\end{theorem}
\begin{proof}
    By~\Cref{lem:optimal-tree:1}, 
    $
    \E[X] \leq \E\lb{ \sum_{v \in Tree^*} P(v) }
    $.
    We can rewrite $P(v)$ as an integral of threshold indicator function of $V(v)$ by substituting $t = - \log x$:
    \[
    P(v)
    = \int_0^{\infty} e^{-t} \mathds{1}\lp{ V(v) \leq t } dt.
    \]
    Let $N(t)$ be a random variable that counts the number of nodes in $\U_{\log}$, whose value is $\leq t$.
    Then, since $Tree^*$ contains $P$ nodes with smallest values:
    \begin{align*}
    \sum_{v \in Tree^*} P(v) 
    &= \int_0^{\infty} e^{-t} \sum_{v \in Tree^*} \mathds{1}\lp{ V(v) \leq t } \, dt \\
    &= \int_0^{\infty} e^{-t} \min \lp{ P, N(t) } \, dt.
    \end{align*}
    Since $N(t) \geq 0$ for all $t\geq 0$, we can swap expectation and integral and then apply Jensen's ($\min \lp{P, \cdot }$ is concave):
    \begin{align*}
    \E \lb{ \int_0^{\infty} e^{-t} \min\lp{P, N(t) } dt } \\
    = \int_0^{\infty} e^{-t} \E \lb{ \min\lp{P, N(t) }} dt \\
    \leq \int_0^{\infty} e^{-t} \min\lp{P, \E \lb{N(t)} } dt.
    \end{align*}
    By~\Cref{lem:ENt-bounds-limit}, $ \E \lb{ N(t) } ={e^t}/{\mu} + o(e^t).$ Substituting gives us:
    \begin{align*}
    \E[X] &\leq  
    \frac{\log P}{\mu} + o \lp{\log P}.        
    \end{align*}
\end{proof}

\paragraph{Lower bound}

Denote $T_P$ as the $P$-th smallest value in $\U_{\log}$.
First, we give a rough asymptotic bound on $\E[T_P]$.
\begin{lemma}\label{lem:topb-bounds-limit}
    Assume $\mu > 0$, $\mu_{(2)} < \infty$, and $\Xi_{\log}$ is not arithmetic.
    Then, 
    \[
    \E \lb{ T_P } \geq
    \log P - \log \log P - O(1).
    \]
\end{lemma}
\begin{proof}
    Throughout this proof, assume $P$ is large enough.
    Using the tail-sum formula for expectations, 
    \begin{align*}
    \E[T_P] &= \int_0^\infty \Pr\lb{T_P > t} dt \\
    &= \int_0^\infty \Pr\lb{N(t) < P} dt.     
    \end{align*}
    Let $t^* = \log(\mu P) - \log \log (P)$.
    \[\int_0^\infty \Pr\lb{N(t) < P} dt \geq t^* \Pr\lb{N(t^*) < P}.  
    \]
    Next, we apply Markov inequality,
    \[
    \Pr\lb{N(t^*) < P} \geq 1-\frac{\E[N(t^*)]}{P}.
    \]
    Finally, by~\Cref{lem:ENt-bounds-limit},
    \[
    \frac{\E[N(t^*)]}{P} \leq \frac{1+o(1)}{\log P}.
    \]
    Which gives us:
    \begin{align*}
    \E[T_P] &\geq t^* \lp{1 - \frac{1+o(1)}{\log P}} \\
    &= \log P - \log \log P - O(1).
    \end{align*}
    
\end{proof}

\begin{lemma}\label{lem:depth-vs-topb}
    Assume $\mu > 0$, $\mu_{(2)} < \infty$, and $\Xi_{\log}$ is not arithmetic.
    \[
    \E[X] \geq \frac{\E[T_P]}{\mu}.
    \]
\end{lemma}
\begin{proof}
Recall the notation from~\Cref{par:brw:U}. 
For $u \in \U_{\log}$, $\beta_u$ is the accepted probability for node $u$, given its parent is accepted.
$P(u)$ is the accepted probability of node $u$.
$V(u) = -\log P(u)$ is the position of $u$ in the branching random walk.

Let $Tree^*$ be some optimal subtree of $\U_{\log}$ of size $P$ (from~\Cref{lem:optimal-tree:1}). 
Let $\delta(Tree^*) = \{u\in \U_{\log} : u \notin Tree^*, parent(u) \in Tree^*\}$ denote the \emph{frontier} of $Tree^*$.
$\delta(Tree^*)$ is a set of all possible candidates for the root of the next iteration.
More precisely, $\{P(u) : u \in \delta(Tree^*\})$ is a distribution, and $P(u)$ is the probability that $u$ is the next token accepted by the verifier $M_p$.
Hence, $\sum_{u \in \delta(Tree^*)} P(u) = 1$.

The maximum value in $Tree^*$ is denoted by $T_P = \max_{v \in Tree^*}(P(v))$.
By construction of $Tree^*$, for every $u \in \delta(Tree^*)$, $P(u) \geq T_P$.
Hence, we have
\[
\sum_{u \in \delta(Tree^*)} P(u)V(u) \geq T_P.
\]

Finally, by Wald's, we get
\[
\E[\sum_{u \in \delta(Tree^*)} P(u) V(u)] =
\mu \E[X].
\]
\end{proof}

As a corollary of~\Cref{lem:topb-bounds-limit,lem:depth-vs-topb} and~\Cref{th:logb_bound_limit}, we get
\begin{corollary}\label{cor:EX-limit}
    Under the assumption of~\Cref{th:logb_bound_limit},
    \[
    \E[X] \approx \frac{\log P}{\mu}.
    \]
\end{corollary}
There is a nice way to interpret the formula. 
A static tree with $P$ nodes and average degree $e^\mu$ would have depth $\E[X]$.
That is, the relation between $e^\mu$ and $\E[X]$ is approximately the same as between the degree of a static balanced tree and its depth. 

\subsection{The Bound for Deterministic Drafter with Imperfect Knowledge}

In previous sections, we assume that the drafter $M_q$ has full knowledge of the target model's output probabilities ($\U$).
This allows us to derive the hard limits for any algorithm that predicts the target model $M_p$.

However, this assumption does not reflect how these systems operate in practice.
Namely, in standard speculative decoding~\citet{leviathan2023fast}, we would normally have a smaller draft model, trained to predict the target model's output.
In turn, this motivates the question: how does the relation between the drafter predictions $q$ and target's distribution $p$ affect the maximum achievable acceptance rate? 
In this section, we analyze the performance of a drafter who only has access to its own distribution $q(x_t | x_{<t})$. 
We derive a lower bound on the expected number of tokens per iteration $\E[X]$, analogously to~\Cref{cor:EX-limit}, with the expected entropy $\mu_{CE} = \E[H(P || Q)] = \E[-\sum p \log q]$.

\paragraph{Model}

In this section, we keep the model from~\Cref{sec:model} unchanged. 
The only modification is the information available to the drafter $M_q$.
Instead of observing the verifiers distribution $p(x_t \mid x_{<t})$, the drafter sees \emph{speculated distributions} $q(x_t \mid x_{<t})$.
In particular, the drafter algorithm is oblivious to the distribution $p$ (even from the previous iteration).

Under~\Cref{assump:iid}, we also assume that $q$ are i.i.d., or formally:
\begin{assumption}[I.i.d. Speculated Distributions]
\label{assump:iid_q}
The speculated distributions $q$ are i.i.d. across different prefixes $q(x_t \mid x_{<t}) = q(x_t)$.
Moreover, the probability that the distribution $q$ is point mass is $ < 1$: $\Pr\lb{q(x_t) = 1} < 1$.
\end{assumption}
Let $\U_q$ denote a token tree with edges labeled with $q$ instead of $p$.
We simply write $q(v)$ to denote the probability on an edge $(u, v)$. 
The \emph{speculated probability} of a node is therefore:
\[
Q(v) = \prod_{u \in \text{path}(r \rightarrow v)} q(u).
\]
The \emph{value} of a node in $\U_q$ is $V_q(v) = -\log(Q(v))$.
The output of the optimal drafter $M_q$ is $Tree_q^*$ --- $Tree^*$ from~\Cref{lem:optimal-tree:1} for $\U_q$.
$T_q(P)$ denotes the $P$-th smallest value in $U_q$.

\paragraph{Lower bound}

Let $\Xi_{(q > 0)} = \{u : |u| = 1,\ q(u) > 0\}$ be a set of nodes in the point process with non-zero probability of speculation.
Let $\mu_{CE}$ denote the expected cross-entropy of the output distribution $p$ and the speculated distribution $q$, conditioned on $q > 0$.
Let $\Pr[q = 0]$ denote the probability of $\{q = 0\}$. Formally,
\begin{align*}
    \mu_{CE} 
        &= \E\!\left[\, - \log q(u) \ \middle|\ q>0 \right] \\[4pt]
        &= \E\!\left[
            \frac{- \sum_{\Xi_{(q > 0)}} p(u)\, \log q(u)}
                 {\sum_{\Xi_{(q > 0)}} p(u)}
        \right], \\[8pt]
    \E\!\left[\Pr[q=0]\right]
        &= \E\!\left[
            \sum_{u:\, |u|=1} p(u)\, \mathds{1}\!\left( q(u)=0 \right)
        \right].
\end{align*}

It is common in speculative decoding to have a drafting model with a smaller vocabulary than the target model.
Hence, we separate two scenarios: a) $p(x) > 0$, while $q(x) = 0$ (usually caused by out-of-vocabulary misses) and $q(x) > 0$ (normal speculative decoding).  
Under this setup, we can prove a lower bound for the expected number of tokens per iteration $\E[X]$:
\begin{lemma}\label{lem:cross_entropy_UB}
    Under~\Cref{assump:iid,assump:iid_q},
\[
\E[X] \ge
\begin{aligned}[t]
&\min\Bigg(
      \frac{1}{\E\big[\Pr[q=0]\big]},\\[4pt]
&\qquad\frac{\log P}{\mu_{CE}} - o(\log P)
\Bigg).
\end{aligned}
\]
\end{lemma}
\begin{proof}
Recall the notation from the proof of~\Cref{lem:depth-vs-topb}. 
$\delta(Tree)$ is the frontier of $Tree$.

First, separate $\E[X]$ into two cases.
Let $\widetilde{w} \in \delta(Tree_q^*)$ denote the last accepted token in this iteration.
Then
\[
\E[X] \geq \min\bigl(\E[X \big| q(\widetilde{w})= 0 ],\! \E[X \big| q(\widetilde{w}) > 0]\bigr).
\]

Case A.
This can be modeled as a series of i.i.d. steps $I_i$.
Let $p_i$ and $q_i$ be two distributions generated in step $i$.
Then, $I_i = \mathds{1}\lp{ q_i(Y_i) = 0 }$, where $Y_i$ is drawn with distribution $p_i$.
Let $\tau_{q=0}$ be the first time $I_i = 1$.
We can use Wald's equation to find its expectation:
\[
\E[\tau_{q=0}] {\E\lb{I_1}} = 1.
\]
Finally, $\E[X \mid q(\widetilde{w}) = 0 ] = \E[\tau_{q=0}]$ and $\E\lb{\Pr[q = 0]} = \E[I_1]$.

Case B.
Similar to~\Cref{lem:depth-vs-topb}, we observed that for every $u \in \delta(Tree_q^*)$, $V_q(u) \geq T_q(P)$.
Then applying Wald's, we get
\[
\E[X \mid q > 0] \geq \frac{\E[T_q(P)]}{\mu_{CE}}.
\]
Finally, with~\Cref{assump:iid_q}, we can apply~\Cref{lem:topb-bounds-limit} to $T_q(P)$ to obtain the desired bound. 
\end{proof}
In practice, $\E\big[\Pr[q=0]\big]$ is usually very small.
This gives the following estimate for $\E[X]$ for speculation without full knowledge:
\[
\frac{\log P}{\mu} \lesssim \E[X] \lesssim \frac{\log P}{\mu_{CE}}.
\]

\section{Experiments}

\begin{table}[ht!]
\centering
\setlength{\tabcolsep}{3pt}
\caption{Values of the parameters $\mu$ and $\mu_{(2)}$  across models and tasks, estimated with temperature $1.0$, for Llama 3.1 8B Instruct, Llama 3.3 70B Instruct, DeepSeek R1 Distill Llama 8B, and Qwen3 8B}
\label{table:mu-C:topp1}
\begin{tabular}{lcccc}
\toprule
 Model & Dataset & Parameters \\
\midrule
L3-8B       & HumanEval & $\mu=0.279, \mu_{(2)}=0.777$  \\
            & MT-bench  & $\mu=1.088, \mu_{(2)}=6.654$  \\
            & GSM8K     & $\mu=0.636, \mu_{(2)}=2.912$  \\
            & CNN/DM    & $\mu=0.577, \mu_{(2)}=1.531$  \\
            & NQ        & $\mu=0.837, \mu_{(2)}=2.926$  \\
            & \textbf{Mean} & $\mu=0.683, \mu_{(2)} = 2.960$\\
\midrule
L3-70B      & HumanEval & $\mu=0.136, \mu_{(2)}=0.238$  \\
            & MT-bench  & $\mu=0.179, \mu_{(2)}=0.383$  \\
            & GSM8K     & $\mu=0.126, \mu_{(2)}=0.199$  \\
            & CNN/DM    & $\mu=0.146, \mu_{(2)}=0.266$  \\
            & NQ        & $\mu=0.179, \mu_{(2)}=0.382$  \\
            & \textbf{Mean} & $\mu=0.153, \mu_{(2)} = 0.29$\\
\midrule
DS-8B       & HumanEval & $\mu=0.530, \mu_{(2)}=1.252$ \\
            & MT-bench  & $\mu=0.633, \mu_{(2)}=1.951$ \\
            & GSM8K     & $\mu=0.233, \mu_{(2)}=0.531$ \\
            & CNN/DM    & $\mu=0.601, \mu_{(2)}=1.650$ \\
            & NQ        & $\mu=0.771, \mu_{(2)}=2.397$ \\
            & \textbf{Mean} & $\mu=0.554, \mu_{(2)} = 1.556$\\
\midrule
Q3-8B       & HumanEval & $\mu=0.178, \mu_{(2)}=0.313$ \\
            & MT-bench  & $\mu=0.369, \mu_{(2)}=0.814$ \\
            & GSM8K     & $\mu=0.202, \mu_{(2)}=0.369$ \\
            & CNN/DM    & $\mu=0.328, \mu_{(2)}=0.691$ \\
            & NQ        & $\mu=0.518, \mu_{(2)}=1.297$ \\
            & \textbf{Mean} & $\mu=0.319, \mu_{(2)} = 0.697$\\
\bottomrule
\end{tabular}
\end{table}

\paragraph{Setup and Objectives.}
To validate our theoretical results, we first evaluate the key expected entropy ($\mu$) and expected second log-moment ($\mu_{(2)}$) for popular models such as Llama 3.1 8B (L3-8B) Instruct and Llama 3.3 70B (L3-70B) Instruct~\citep{grattafiori2024llama}, DeepSeek R1 Distill Llama 8B (DS-8B)~\citep{deepseek2025}, and Qwen3 8B (Q3-8B)~\citep{qwen3technicalreport2025}.
We evaluate these parameters across several task types: code generation, conversation, math problem solving, summarization, and question answering. For robustness, we chose the benchmarks used by the EAGLE~\citet{li2024eagle} series of papers: HumanEval~\citep{chen2021evaluating}, MT-bench~\citep{zheng2023judging}, GSM8K~\citep{cobbe2021training}, CNN/Daily Mail~\citep{nallapati2016abstractive}, and Natural Questions (NQ)~\citep{kwiatkowski2019natural}. 
Full experiments take a few hours on a standard multi-GPU server. 

\paragraph{Entropy measurements.}
~\Cref{table:mu-C:topp1} shows the estimated expected entropies ($\mu$) and expected second log-moments ( $\mu_{(2)}$ ), for~\Cref{th:logb_bound}.
We compute these values for each model and benchmark with temperature $1$ for all runs. 
The parameters are computed over the output distributions of the target model running on the corresponding dataset ($80$ samples per dataset). 
When substituted in~\Cref{th:logb_bound}, these results provide evaluations for which models and tasks offer the highest potential for parallelization.

The results in~\Cref{table:mu-C:topp1} show both lower and more stable entropy and second-moment bounds for the larger and more accurate Llama 3.3 70B model, across all tasks. By contrast, the 8B models have both less stable and higher parameter values, across all tasks, but especially so in the multi-turn MT-Bench benchmark. 
In terms of our bounds, this implies a higher parallelization potential for the larger model. 

Across different architectures, we observe that Qwen3 consistently exhibits a lower expected entropy ($\mu$) across almost all benchmarks compared to Llama 3.1. According to our bounds, this suggests that Qwen3 is better suited for speculation.

\begin{figure}[ht!]
    \centering
    \includegraphics[width=\columnwidth]{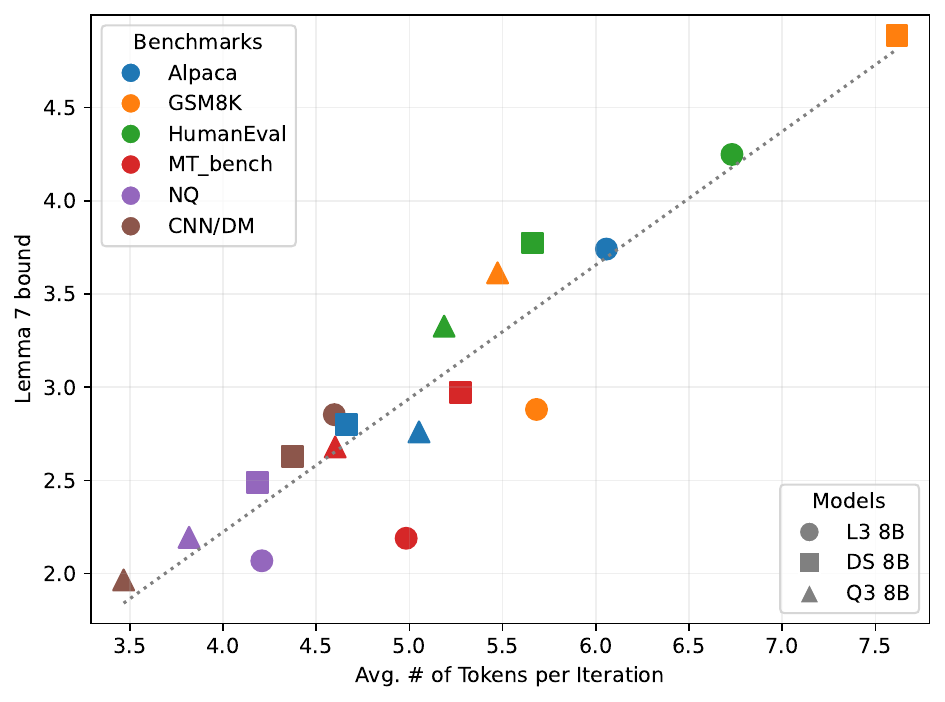}
    \caption{Validating~\Cref{lem:cross_entropy_UB} with EAGLE-3 across different models. The $\pm o(\log P)$ term is neglected for this plot. Speculation size is set to $P=60$. The interpolation line shows the near-linear correspondence between the EAGLE speedup results and lower bound. }
    \label{fig:8B-eagle-vs-CE-bound-correlation}
\end{figure}

\paragraph{Upper bound VS EAGLE-3}
Next, in~\Cref{fig:8B-eagle-vs-CE-bound-correlation}, we evaluate how well~\Cref{lem:cross_entropy_UB}, our lower bound on expected accepted tokens, captures the expected number of tokens successfully speculated by EAGLE-3 across different 8B models. 
For each iteration, we run EAGLE-3 with the sampling method described in~\Cref{lem:optimal-tree:1}, which corresponds to the original EAGLE-2 sampling procedure, but without a limit on the depth and width of the tree.
This illustrates the potential of the algorithm if we neglect the speculation time ($\tau$ from~\citet{li2025eagle3}).
We fix the speculation size to $P = 60$.
The plot shows a clear linear relationship between our bound and the EAGLE-3 results, indicating that actual speedup scales predictably with the~\Cref{th:logb_bound} bound across dataset.
However, the bound is smaller, indicating 
a constant gap between our bound and the actual performance of EAGLE-3. 
This is likely the $\pm o(\log P)$ term.
\begin{figure*}[ht!]
    \centering
    \includegraphics[width=\textwidth]{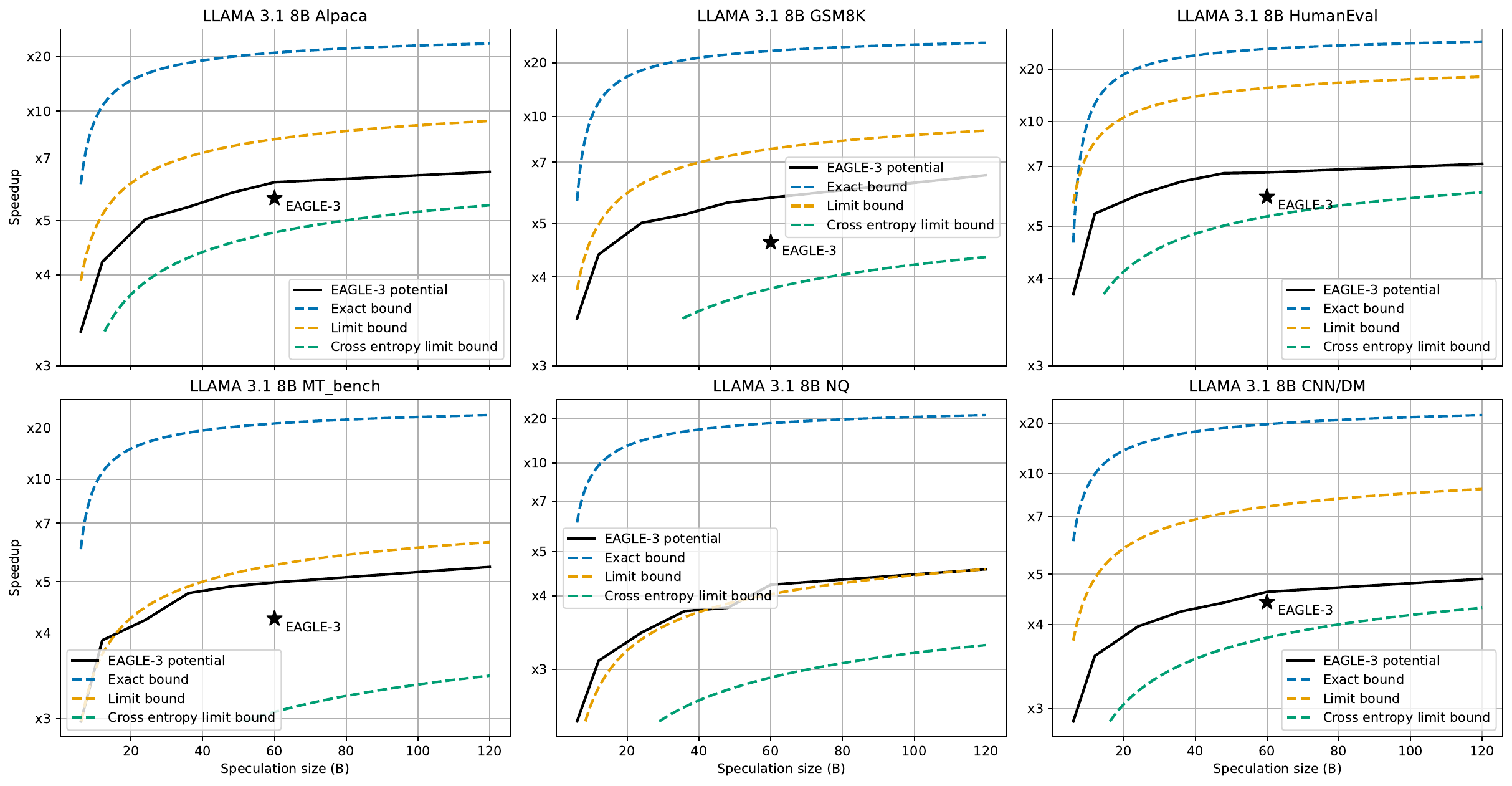}
    \caption{Comparing~\Cref{cor:EX-limit,lem:cross_entropy_UB} with (EAGLE-3, LLaMA 3.1 Instruct (8B)) the $\pm o(\log P)$ term is neglected for this plot. The blue line represent the precise bound from\Cref{th:logb_bound}}
    \label{fig:8B-eagle-vs-limit-bound}
\end{figure*}

Finally, ~\Cref{fig:8B-eagle-vs-limit-bound} shows how the speedup achieved by the EAGLE-3 algorithm varies with the speculation size $P$ when speculating on the LLaMA 3.1 8B Instruct, on $6$ different benchmarks. The solid black line represents the best achievable speedup for EAGLE-3 under our simplifying assumptions. Specifically, we compute this bound by assuming a negligible speculation overhead, but use the actual EAGLE-3 acceptance rate, given their trained speculators. Further, the dashed lines corresponds to the theoretical bounds on maximal speedup. Specifically, the blue lines shows the exact upper bounds from~\Cref{th:logb_bound}; the orange lines represent the limit bounds ($P \rightarrow \infty$) established in~\Cref{cor:EX-limit}; and the green line shows the lower speedup limit for the imperfect knowledge drafting of~\Cref{lem:cross_entropy_UB} (also for $P \rightarrow \infty$). 
For both orange and green lines, we ignore the $o(\log P)$ error, and hence these graphs should be viewed with caution. The star marker denotes the performance of EAGLE-3 according to~\citet{li2025eagle3} (configured for practical use). The star is missing from a Natural Questions benchmark since it was not presented in the  EAGLE-3 paper. 

First, we notice that the gap between the real EAGLE-3 performance under realistic condition (star point) and the potential bound (solid black line) is relatively low. However, EAGLE-3 still leaves a large gap to the theoretical exact lower bound, represented by the blue dashed lines. 
Specifically, we postulate that this is because of the gap between the optimal speculator adopted in our analytical argument, and the one adopted by the EAGLE-3 implementation. 
Secondly, the gap between the limit line (orange) and the imperfect knowledge line (green) is mostly due to an oracle speculator assumption.
Generally, this suggests that about $\times 2$ difference between the theoretical bound and EAGLE-3 is caused by a speculation error.
This suggests that there is still significant room for improvement with regards to designing optimal practical speculation algorithms.

\section{Discussion}

We established theoretical limits on the speedup achievable by deterministic speculative decoding algorithms, by drawing a novel connection to the theory of Branching Random Walks (BRW). 
Our main result (\Cref{th:logb_bound}) shows that the expected number of accepted tokens per iteration scales logarithmically with the verification capacity $P$, which in turn implies significant diminishing returns from increasing system parallelism. Specifically, exponentially increasing the computational budget $P$ will only yield a linear improvement in speedup, highlighting that massive parallelism alone cannot overcome the probabilistic bottlenecks of speculation.

Roughly, the bound is inversely related to the expected entropy $\mu$ of the target model. This quantifies the intuition that speculation is inherently difficult when the model's output is uncertain. Furthermore, the constant factor $(\mu + \mu_{(2)})/\mu^2$ closely bounds the impact of variability in the entropy distribution. More precisely, a higher expected second log-moment $\mu_{(2)}$ (greater variance in log-probabilities) reduces the maximum achievable speedup, indicating that the specific distribution of uncertainty impacts performance beyond just the average entropy.
More generally, modeling the token tree of log-probabilities as a BRW provides a  new analytical lens for understanding speculative generation. This  allows us to rigorously analyze the distribution of high-probability paths under resource constraints, moving beyond previous simplified models based on average acceptance rates.

Finally, the empirical evaluations demonstrate a strong correlation between our theoretical bounds and the performance of state-of-the-art systems like EAGLE-3. This suggests that current deterministic algorithms are approaching the fundamental limits under the assumptions studied, and can guide future system design, emphasizing the importance of model characteristics (reducing effective entropy) over merely increasing the parallelism budget $P$.

\section{Limitations}

Our theoretical framework relies on a number of simplifying assumptions.

\paragraph{I.i.d. Assumption.} The analysis depends on \Cref{assump:iid}, which states that the distributions of acceptance probabilities are independent and identically distributed across prefixes. This is not exact for language models, where context dependency is inherent and output entropy will vary   with the preceding text. This assumption is needed for the application of BRW theory and Wald's equation. Extending the analysis to stationary and ergodic processes, which better capture the dynamics of language modeling, remains an important direction for future work.

\paragraph{Simplified Timing Model.} We adopt a simplified timing model (\Cref{assump:timing}) that assumes a constant latency $T$ for the verifier and negligible computational cost for the drafter. In practice, the verifier's latency increases with the prefix length (due to KV cache growth); at the same time, highly-efficient drafting mechanisms such as the one in EAGLE-3 do have negligible overhead. These assumptions allow us to isolate the fundamental limits imposed by the probabilistic verification process, but abstract away practical engineering trade-offs.

\paragraph{Perfect Knowledge and Deterministic Drafting.} Our upper bounds are derived for the optimal deterministic drafting strategy, assuming that the drafter has perfect knowledge of the target model's output probabilities ($\U$). This is done for analytical purposes, as we wish to bound the performance of an ideal process. Yet, in reality, drafters only approximate the target distribution. Furthermore, the analysis is focused on deterministic drafting and does not address stochastic speculation strategies.

\section{LLM Use}

LLMs were only used for minor editing and proofreading of the draft. 

\bibliography{references}

@inproceedings{cai2401medusa,
  author    = {Tianle Cai and Yuhong Li and Zhengyang Geng and Hongwu Peng and Jason D. Lee and Deming Chen and Tri Dao},
  title     = {Medusa: Simple {LLM} Inference Acceleration Framework with Multiple Decoding Heads},
  booktitle = {International Conference on Machine Learning},
  year      = {2024},
  organization = {PMLR}
}

@inproceedings{li2024eagle, 
	author = {Yuhui Li and Fangyun Wei and Chao Zhang and Hongyang Zhang}, 
	title = {{EAGLE}: Speculative Sampling Requires Rethinking Feature Uncertainty}, 
	booktitle = {International Conference on Machine Learning},
	year = {2024}
}

@inproceedings{li2024eagle2, 
	author = {Yuhui Li and Fangyun Wei and Chao Zhang and Hongyang Zhang}, 
	title = {{EAGLE-2}: Faster Inference of Language Models with Dynamic Draft Trees}, 
	booktitle = {Empirical Methods in Natural Language Processing},
	year = {2024}
}

@inproceedings{li2025eagle3,
    author = {Yuhui Li and Fangyun Wei and Chao Zhang and Hongyang Zhang},
    title = {{EAGLE-3}: Scaling up Inference Acceleration of Large Language Models via Training-Time Test}, 
    booktitle = {Annual Conference on Neural Information Processing Systems},
    year = {2025}
}

@inproceedings{leviathan2023fast,
  title={Fast inference from transformers via speculative decoding},
  author={Leviathan, Yaniv and Kalman, Matan and Matias, Yossi},
  booktitle={International Conference on Machine Learning},
  pages={19274--19286},
  year={2023},
  organization={PMLR}
}

@inproceedings{sheng2023flexgen,
  title={FlexGen: High-Throughput Generative Inference of Large Language Models with a Single {GPU}},
  author={Sheng, Ying and Zheng, Lianmin and Yuan, Binhang and Li, Zhuohan and Ryabinin, Max and Chen, Beidi and Liang, Percy and R{\'e}, Christopher and Stoica, Ion and Zhang, Ce},
  booktitle={International Conference on Machine Learning},
  pages={31094--31116},
  year={2023},
  organization={PMLR}
}

@inproceedings{LiuLLLZHS25_pearl,
  author       = {Liu, Tianyu and Li, Yun and Lv, Qitan and Liu, Kai and Zhu, Jianchen and Hu, Winston},
  title        = {{PEARL:} Parallel Speculative Decoding with Adaptive Draft Length},
  booktitle    = {International Conference on Learning Representations},
  year         = {2025},
  organization = {OpenReview}
}

@article{chen2023accelerating,
  title={Accelerating Large Language Model Decoding with Speculative Sampling},
  author={Chen, Charlie and Borgeaud, Sebastian and Irving, Geoffrey and Lespiau, Jean-Baptiste and Sifre, Laurent and Jumper, John},
  journal={arXiv preprint arXiv:2302.01318},
  year={2023}
}

@inproceedings{stern2018blockwise,
  title={Blockwise Parallel Decoding for Deep Autoregressive Models},
  author={Stern, Mitchell and Shazeer, Noam and Uszkoreit, Jakob},
  booktitle={Advances in Neural Information Processing Systems},
  year={2018}
}

@inproceedings{kwon2023efficient,
  title={Efficient Memory Management for Large Language Model Serving with PagedAttention},
  author={Kwon, Woosuk and Li, Zhuohan and Zhuang, Siyuan and Sheng, Ying and Zheng, Lianmin and Yu, Cody Hao and Gonzalez, Joseph and Zhang, Hao and Stoica, Ion},
  booktitle={Proceedings of the 29th symposium on operating systems principles},
  pages={611--626},
  year={2023}
}

@inproceedings{zheng2024sglang,
  title={SGLang: Efficient Execution of Structured Language Model Programs},
  author={Zheng, Lianmin and Yin, Liangsheng and Xie, Zhiqiang and Sun, Chuyue Livia and Huang, Jeff and Yu, Cody Hao and Cao, Shiyi and Kozyrakis, Christos and Stoica, Ion and Gonzalez, Joseph E and others},
  booktitle={Advances in neural information processing systems},
  pages={62557--62583},
  year={2024}
}

@inproceedings{wang2024theoretical,
  title={A Theoretical Perspective for Speculative Decoding Algorithm},
  author={Yin, Ming and Chen, Minshuo and Huang, Kaixuan and Wang, Mengdi},
  booktitle={Advances in Neural Information Processing Systems},
  pages={128082--128117},
  year={2024}
}

@article{shazeer2019fast,
  title={Fast Transformer Decoding: One Write-Head is All You Need},
  author={Shazeer, Noam},
  journal={arXiv preprint arXiv:1911.02150},
  year={2019}
}

@article{miao2023specinfer,
  title={SpecInfer: Accelerating Generative {LLM} Serving with Speculative Inference and Token Tree Verification},
  author={Miao, Xupeng and Oliaro, Gabriele and Zhang, Zhihao and Cheng, Xinhao and Wang, Zeyu and Wong, Rae Ying Yee and Chen, Zhuoming and Arfeen, Daiyaan and Abhyankar, Reyna and Jia, Zhihao},
  journal={arXiv preprint arXiv:2305.09781},
  year={2023}
}

@article{spector2023accelerating,
  title={Accelerating {LLM} Inference with Staged Speculative Decoding},
  author={Spector, Benjamin and Re, Chris},
  journal={arXiv preprint arXiv:2308.04623},
  year={2023}
}

@inproceedings{elzinga2024layerskip,
  title={LayerSkip: Enabling Early Exit Inference and Self-Speculative Decoding},
  author={Elhoushi, Mostafa and Shrivastava, Akshat and Liskovich, Diana and Hosmer, Basil and Wasti, Bram and Lai, Liangzhen and Mahmoud, Anas and Acun, Bilge and Agarwal, Saurabh and Roman, Ahmed and others},
  booktitle={Proceedings of the 62nd Annual Meeting of the Association for Computational Linguistics (Volume 1: Long Papers)},
  pages={12622--12642},
  year={2024}
}

@inproceedings{fu2024break,
  title={Break the Sequential Dependency of LLM Inference Using Lookahead Decoding},
  author={Fu, Yichao and Bailis, Peter and Stoica, Ion and Zhang, Hao},
  booktitle={International Conference on Machine Learning},
  year={2024}
}

@inproceedings{sun2023spectr,
  title={SpecTr: Fast Speculative Decoding via Optimal Transport},
  author={Sun, Ziteng and Suresh, Ananda Theertha and Ro, Jae Hun and Beirami, Ahmad and Jain, Himanshu and Yu, Felix},
  booktitle={Advances in Neural Information Processing Systems},
  year={2023}
}

@article{biggins1977martingale,
  title={Martingale Convergence in the Branching Random Walk},
  author={Biggins, John D},
  journal={Journal of Applied Probability},
  volume={14},
  number={1},
  pages={25--37},
  year={1977},
  publisher={Cambridge University Press}
}

@book{shi2015branching,
  title={Branching Random Walks},
  author={Shi, Zhan},
  series={Lecture Notes in Mathematics},
  volume={2151},
  year={2015},
  publisher={Springer}
}

@book{lyons2017probability,
  title={Probability on Trees and Networks},
  author={Lyons, Russell and Peres, Yuval},
  publisher={Cambridge University Press},
  volume={42},
  year={2017}
}

@article{grattafiori2024llama,
  title= {The Llama 3 Herd of Models},
  author={Grattafiori, Aaron and Dubey, Abhimanyu and Jauhri, Abhinav and Pandey, Abhinav and Kadian, Abhishek and Al-Dahle, Ahmad and Letman, Aiesha and Mathur, Akhil and Schelten, Alan and Vaughan, Alex and others},
  journal={arXiv preprint arXiv:2407.21783},
  year={2024}
}

@article{chen2021evaluating,
  title={Evaluating Large Language Models Trained on Code},
  author={Chen, Mark and Tworek, Jerry and Jun, Heewoo and Yuan, Qiming and Pinto, Henrique Ponde De Oliveira and Kaplan, Jared and Edwards, Harri and Burda, Yuri and Joseph, Nicholas and Brockman, Greg and others},
  journal={arXiv preprint arXiv:2107.03374},
  year={2021}
}

@article{zheng2023judging,
  title={Judging LLM-as-a-Judge with MT-Bench and Chatbot Arena},
  author={Zheng, Lianmin and Chiang, Wei-Lin and Sheng, Ying and Zhuang, Siyuan and Wu, Zhanghao and Zhuang, Yonghao and Lin, Zi and Li, Zhuohan and Li, Dacheng and Xing, Eric and others},
  journal={Advances in neural information processing systems},
  volume={36},
  pages={46595--46623},
  year={2023}
}

@article{cobbe2021training,
  title={Training verifiers to solve math word problems},
  author={Cobbe, Karl and Kosaraju, Vineet and Bavarian, Mohammad and Chen, Mark and Jun, Heewoo and Kaiser, Lukasz and Plappert, Matthias and Tworek, Jerry and Hilton, Jacob and Nakano, Reiichiro and others},
  journal={arXiv preprint arXiv:2110.14168},
  year={2021}
}

@article{nallapati2016abstractive,
  title={Abstractive text summarization using sequence-to-sequence rnns and beyond},
  author={Nallapati, Ramesh and Zhou, Bowen and Gulcehre, Caglar and Xiang, Bing and others},
  journal={arXiv preprint arXiv:1602.06023},
  year={2016}
}

@article{kwiatkowski2019natural,
  title={Natural questions: a benchmark for question answering research},
  author={Kwiatkowski, Tom and Palomaki, Jennimaria and Redfield, Olivia and Collins, Michael and Parikh, Ankur and Alberti, Chris and Epstein, Danielle and Polosukhin, Illia and Devlin, Jacob and Lee, Kenton and others},
  journal={Transactions of the Association for Computational Linguistics},
  volume={7},
  pages={453--466},
  year={2019},
  publisher={MIT Press One Rogers Street, Cambridge, MA 02142-1209, USA journals-info~…}
}

@book{grimmett2020probability,
  title={Probability and random processes},
  author={Grimmett, Geoffrey and Stirzaker, David},
  year={2020},
  publisher={Oxford university press}
}

@misc{deepseek2025,
      title={DeepSeek-R1: Incentivizing Reasoning Capability in LLMs via Reinforcement Learning}, 
      author={DeepSeek-AI},
      year={2025},
      eprint={2501.12948},
      archivePrefix={arXiv},
      primaryClass={cs.CL},
      url={https://arxiv.org/abs/2501.12948}, 
}

@misc{qwen3technicalreport2025,
      title={Qwen3 Technical Report}, 
      author={Qwen Team},
      year={2025},
      eprint={2505.09388},
      archivePrefix={arXiv},
      primaryClass={cs.CL},
      url={https://arxiv.org/abs/2505.09388}, 
}

@article{lorden1970excess,
  title={On excess over the boundary},
  author={Lorden, Gary},
  journal={The Annals of Mathematical Statistics},
  volume={41},
  number={2},
  pages={520--527},
  year={1970},
  publisher={Institute of Mathematical Statistics}
}

\appendix

\section{Renewal bound}\label{app:renewal}
Let $X_1, X_2, \ldots$ be i.i.d. positive random variables and define the partial sums
$S_d = \sum_{i=1}^d X_i$ with $S_0 = 0$.  
For $x \geq 0$, set
$U(x) = \sum_{d=0}^\infty \Pr[S_d \leq x]$.
\begin{claim}
For all $x \geq 0$,
\[
\frac{x}{\E[X_1]} + 1 \;\;\leq\;\; U(x) \;\;\leq\;\; \frac{x}{\E[X_1]} + 1 + \frac{\E[X_1^2]}{(\E[X_1])^2}.
\]
\end{claim}
\begin{proof}
Let $\widetilde{N}_{\geq}(x) = \min\{d \geq 0 : S_d \geq x\}$ denote the number of steps to reach $x$.  
Then
\begin{align*}
U(x) &= \sum_{d=0}^\infty \Pr[S_d \leq x] \\
    &= \sum_{d=0}^\infty \Pr[\widetilde{N}_{\geq}(x) \geq d]
      = 1 + \E[\widetilde{N}_{\geq}(x)].
\end{align*}
By Wald’s identity,
\[
\E[\widetilde{N}_{\geq}(x)] = \frac{\E[S_{\widetilde{N}_{\geq}(x)}]}{\E[X_1]}
   \;\;\geq\;\; \frac{x}{\E[X_1]}.
\]
Lorden’s inequality~\citep{lorden1970excess} gives
\[
\E[\widetilde{N}_{\geq}(x)] \;\leq\; \frac{x}{\E[X_1]} + \frac{\E[X_1^2]}{(\E[X_1])^2}.
\]
Combining the two bounds completes the proof.
\end{proof}

\end{document}